%% file: templateArxiv.tex
\documentclass{article}

\usepackage{PRIMEarxiv}
\input{macro.tex}

\usepackage{natbib}
\usepackage{amsmath}
\usepackage{amsthm}
\usepackage{amssymb}
\usepackage{algorithm}

\usepackage{amsmath}
\usepackage{amssymb}
\usepackage{mathtools}
\usepackage{amsthm}
\usepackage{accents}
\usepackage{bbm}
\usepackage{multicol,lipsum,xparse}

\DeclareMathOperator{\err}{err}

\usepackage[utf8]{inputenc} 
\usepackage[T1]{fontenc}    
\usepackage[hidelinks]{hyperref}
\usepackage{url}            
\usepackage{booktabs}       
\usepackage{amsfonts}       
\usepackage{nicefrac}       
\usepackage{microtype}      
\usepackage{lipsum}
\usepackage{fancyhdr}       
\usepackage{graphicx}       
\graphicspath{{media/}}     
\usepackage{graphicx}
\usepackage{subcaption}

\theoremstyle{plain}

\title{Online Conformal Prediction: Enforcing monotonicity via Online Optimization}

\author{
  Eduardo Ochoa Rivera, Ambuj Tewari \\
  University of Michigan, \\
  Ann Arbor, USA \\
  \today
}

\begin{document}
\maketitle

\begin{abstract}


Conformal prediction provides a principled framework for uncertainty quantification with finite-sample coverage guarantees. While recent work has extended conformal prediction to online and sequential settings, existing methods typically focus on a single coverage level and do not ensure consistency across multiple confidence levels. In many real-world applications, such as weather forecasting, macroeconomic prediction, and risk management, different users operate under heterogeneous risk tolerances and require calibrated uncertainty estimates across a range of coverage levels. In such settings, it is desirable to produce prediction sets corresponding to different coverage levels that are nested and valid simultaneously. In this paper, we propose two novel online conformal prediction methods that output \emph{nested prediction sets} across a range of coverage levels, enabling simultaneous uncertainty quantification across the entire risk spectrum. Beyond interpretability, jointly estimating multiple coverage levels is known to improve statistical efficiency in classical quantile regression by enforcing non-crossing constraints and sharing information across quantiles.
Our approaches leverage an online optimization perspective with small regret that translates to quantile estimation error control while enforcing nestedness of prediction sets. Empirical results on synthetic and real-world datasets, including applications in forecasting tasks with heterogeneous risk requirements, demonstrate that our method achieves stable coverage across all levels, strictly nested prediction sets, and improved efficiency compared to existing online conformal baselines.

\end{abstract}


\section{Introduction}\label{section:intro}

Conformal prediction has emerged as a powerful tool for constructing prediction sets with distribution-free coverage guarantees \citep{papadopoulos2002inductive, vovk2005algorithmic, shafer2008tutorial}. Given a desired coverage level $1-\alpha$, conformal methods output a prediction set that contains the true outcome with probability at least $1-\alpha$ under minimal assumptions. These guarantees have made conformal prediction attractive in safety-critical and risk-sensitive applications \citep{lu2022fair, lu2023federated, kuchibhotla2023nested, lindemann2023safe, clough2024enhancing}.

Many modern applications, however, operate in an online or streaming regime, where data arrives sequentially and predictions must be produced in real time \citep{chernozhukov2018exact, gibbs2021adaptive, zaffran2022adaptive}. This has motivated the development of online conformal prediction methods, which aim to maintain validity while update predictions as new data is observed. In contrast to the batch setting, where nestedness of prediction sets naturally arises for standard nonconformity scores, this property is generally lost in the online setting. To the best of our knowledge, existing online conformal prediction methods operate at a single coverage level and must be rerun independently for different values of $\alpha$, resulting in prediction sets that are not necessarily nested. This limitation restricts their applicability in settings where uncertainty must be assessed across multiple risk levels simultaneously.
 
In practice, decision-makers often require prediction sets at multiple confidence levels simultaneously, reflecting heterogeneous risk preferences across users or decision contexts. For example, in hurricane forecasting, emergency planners may require different levels of conservativeness depending on location and evacuation costs \citep{regnier2008public}. Similarly, in macroeconomic forecasting, policymakers rely on predictive distributions that reflect varying degrees of uncertainty, corresponding to different implicit risk levels \citep{dowd2007too}. In such cases, it is natural to require that prediction sets are \emph{nested}, so that higher coverage levels correspond to larger sets. However, direct application of current online conformal prediction to multiple $\alpha$ values do not guarantee this property.

Beyond interpretability, there is also a strong statistical motivation for jointly estimating multiple coverage levels. In the quantile regression literature, estimating quantiles independently can lead to suboptimal performance and violations of monotonicity (quantile crossing). Joint estimation methods that enforce non-crossing constraints produce coherent quantile functions and improve statistical efficiency by sharing information across quantiles \citep{zou2008composite, liu2011simultaneous}. This suggests that, in the online conformal setting, jointly update multiple quantiles may similarly improve stability and efficiency, compared to running independently for each coverage level.

This paper addresses these challenges by proposing two novel online conformal prediction methods that produce nested prediction sets across a range of coverage levels while preserving long-run coverage guarantees. 

\paragraph{Contributions.}
The main contributions of this paper are as follows:
\begin{itemize}
    \item We introduce a novel formulation of online conformal prediction that produces \emph{nested prediction sets} across multiple coverage levels, enabling calibrated uncertainty quantification over the entire risk spectrum.
    
    \item We propose two online algorithms based on convex optimization algorithms that jointly estimate multiple conformal thresholds while guaranteeing monotonicity: exponentiated gradient and projected gradient. 
    
    \item We establish theoretical guarantees for long-run coverage across all coverage levels, together with guarantees ensuring nested prediction sets at every time step.
    
    \item We demonstrate empirically, on both synthetic and real-world datasets, that our methods achieve improved calibration, efficiency, and stability compared to independent single-level conformal approaches.
\end{itemize}

\section{Related Work}\label{section:related}


Conformal prediction has traditionally relied on exchangeability assumptions to obtain finite-sample coverage guarantees \citep{vovk2005algorithmic, shafer2008tutorial}. Under exchangeability, prediction sets can be constructed using empirical quantiles of past nonconformity scores, and nestedness across coverage levels follows naturally when using an appropriate nonconformity score function. However, these guarantees break down in online settings where the data stream may exhibit temporal dependence, distribution shift, or adversarial behavior.

To address this limitation, recent work has extended conformal prediction to the online setting without relying on exchangeability. A key line of work, initiated by \cite{gibbs2021adaptive}, proposes update the conformal threshold $q_t$ via online gradient descent with a fixed step size known as Adaptive Conformal Inference (ACI). This approach provides \emph{long-run coverage guarantees} under arbitrary (potentially adversarial) data sequences. Subsequent works have proposed refinements to improve practical performance, particularly in terms of prediction set size and adaptivity to distributional changes \citep{bhatnagar2023improved, angelopoulos2023conformal, gibbs2024conformal,areces2025online, sun2025conformal}.

More recent developments have explored the use of time-varying or adaptive step sizes to better handle non-stationarity and to bridge adversarial and stochastic regimes. These methods often draw on tools from online learning, such as adaptive regret minimization and expert aggregation, and aim to provide stronger notions of validity beyond long-run coverage \citep{zaffran2022adaptive, bastani2022practical}. Notably, \cite{angelopoulos2024online} and related works establish guarantees that hold simultaneously in adversarial and i.i.d. settings, achieving a “best-of-both-worlds” behavior.

Parallel to this line of work, several approaches have studied conformal prediction under relaxed assumptions on the data-generating process, such as weak dependence or time-series structure \citep{chernozhukov2018exact, barber2023conformal}. While these methods can provide valid coverage under specific assumptions, they generally do not offer distribution-free guarantees in fully adversarial settings.

Despite this progress, existing online conformal prediction methods are designed to operate at a \emph{single} coverage level. When multiple coverage levels are required, standard approaches must be run independently for each $\alpha$, which does not guarantee nested prediction sets and can lead to inefficiencies. In addition, the quantile regression literature has shown that joint estimation of multiple quantiles, with constraints to prevent crossing, improves both coherence and statistical efficiency \citep{zou2008composite, liu2011simultaneous}.

Our work builds on the online optimization perspective of conformal prediction and extends it to the joint estimation of multiple coverage levels. By enforcing structural constraints across quantiles, we obtain nested prediction sets and calibrated uncertainty estimates over the full risk spectrum, a setting that has not been addressed in prior online conformal prediction work.

\section{Setting}\label{section:setting}

Consider a bounded score function $s_t: \cX \times \cY \rightarrow [0, B]$ at each time $t$ and a stream of data $\{(X_t,Y_t)\}_{t=1}^T$. The score functions come from a pre-trained or trained online model (e.g. $\hat{f}_t$) that predicts $y$ using $x$. The function $s_t(x, y)$, often times called non-conformity score, measures how well the value $y$ \textit{conforms} with the predictions of our fitted model. A common example in the regression setting is $s_t(x, y) = |y-\hat{f}_t(x)|$. Given a coverage level $1-\alpha$ and a threshold $q_t(1-\alpha)$, we can define the prediction set as

\begin{equation}\label{eq:pred_set}
    \cC^{1-\alpha}_t(x) = \{y \in \cY: s_t(x, y) \le q_t(1-\alpha) \}
\end{equation} 

Under the exchangeability of the data, and symmetry of the given model fitting algorithm as a function of the data, this definition of prediction sets, with $q_t(1-\alpha)$ the $\frac{[(n+1)(1-\alpha)]}{n} $ quantile of the calibration scores $\{s(X_t,Y_t)\}_{t=1}^T$, the following holds:

$$
\bbP[Y_t \in \mathcal{C}^{1-\alpha}_t\left(X_t\right)] \ge 1-\alpha
$$

However, when these assumptions do not hold, we lose this appealing guarantee. \cite{gibbs2021adaptive} introduce a new method to set the threshold $q_t$, namely ACI. Currently, most of the online conformal prediction algorithms update the threshold $q_t(1-\alpha)$ in an online fashion:

\begin{equation}\label{eq:update_q}    q_{t+1}(1-\alpha)=q_t(1-\alpha)-\eta_t\left(\alpha-\err_t(\alpha)\right), \qquad \err_t(\alpha) = \mathbbm{1}_{Y_t \notin \mathcal{C}^{1-\alpha}_t\left(X_t\right)}
\end{equation}

Note that $\err_t(\alpha)$, the $\alpha$ miscoverage error, is equivalent to $\mathbbm{1}_{s_t > q_t(1-\alpha)}$. As pointed out in previous works, the update rule in Eq.~\eqref{eq:update_q}  does something intuitive, adjusts the threshold $q_t(1-\alpha)$ based on the discrepancy between observed and target miscoverage $\alpha$. Some authors have referred to this update rule as the quantile tracker \citep{angelopoulos2023conformal, angelopoulos2024online}. It has been shown that, although simple, it has long run coverage guarantees, meaning 

$$
\frac{1}{T} \sum_{t=1}^T \mathbbm{1}_{Y_t \in \mathcal{C}_t\left(X_t\right)} \rightarrow 1-\alpha
$$

However, this definition has a subtle caveat. If we run it for different levels of coverage, the prediciton sets are not necessarily nested, which it is undesirable for applications that multiple coverage levels are needed at the same time. Formally, consider an increasing sequence of miscoverage levels $0=\alpha_0<\alpha_1<\alpha_2< \dots <\alpha_K<\alpha_{K+1}=1$. To simplify notation, we denote $\mathcal{C}^{i}_t := \mathcal{C}^{1-\alpha_i}_t\left(X_t\right)$, $q_{t, i} := q_t(1-\alpha_i)$ and $\err_{t,i} := \err_t(\alpha_i)$. 
For every $i \in [K]$, our goal is to control miscoverage while preserving monotonicity in the thresholds $q_{t, i} > q_{t, i+1}$ for all $t\ge1$. In other words, nestedness in the prediction sets $\mathcal{C}^{i+1}_t \subset \mathcal{C}^{i}_t$.

\section{Methods}\label{section:method}


We explore alternative approaches to enforce nestedness of prediction sets across increasing miscoverage levels. Our starting point is the standard online update of quantile thresholds, which it can also be interpreted through an online optimization lens. In particular, the update in Equation \eqref{eq:update_q} can be viewed as an instance of online (sub)gradient descent applied to the pinball (quantile) loss $\rho_t(q, 1-\alpha)=\left(s_t-q\right)\left(\mathbbm{1}_{\left\{s_t>q\right\}}-\alpha\right)$,

\begin{equation}\label{eq:update_q_grad}    q_{t+1}=q_t-\eta_t\nabla\rho_t(q_t, 1-\alpha)
\end{equation}

We can further extend this definition to estimate quantiles for multiple miscoverage levels. Let's define the join loss function as $f_t(\bar{q}) = \sum_{i=1}^K \rho_t(q_i, 1-\alpha_i)$ where $\bar{q} = (q_1, \dots, q_K)$. We can then express the joint update rule as 

\begin{equation}\label{eq:update_q_multi_grad}    
\bar{q}_{t+1}=\bar{q}_t-\eta_t\nabla f_t(\bar{q}_t)
\end{equation}

Note that this update rule is just individually estimating different quantile at the same time, but the estimations are not sharing information and they won't necessarily be monotonic. Using the common online optimization perspective, we can impose constraints that will modify the update rule to guarantee monotonicity.
Let,s revisit the definition of the online mirror descent algorithm

$$
x_{t+1} = \arg \min_{x \in \cX} \eta \langle \nabla f(x_t), x \rangle + D_{\phi}(x,x_t)
$$

Where $D_{\phi}$ is the Bregman divergence generated by a convex function $\phi$. When $\phi(x) = \frac{1}{2} \|x\|_2^2$ and $\cX = \bbR^n$, we recover the usual gradient descent algorithm and then the update rule in Eq. \eqref{eq:update_q_multi_grad}. If we constrain $\cX$ to be a convex set (for example, the set of decreasing sequences in $\bbR^n$) we are under the projected gradient descent (PG). 
$$
x_{t+1} = \Pi_{\cX} \left(x_t- \eta \nabla f(x_t)\right)
$$
Where $\Pi_{\cX}$ denotes the projection on the the convex set $\cX$. Finally, if we set $\phi(x) = \sum_{i=1}^n x_i \log (x_i)$ and $\cX =\{ x \in \mathbb{R}^n \mid x_i \geq 0, \sum_{i=1}^n x_i=1\}$ we end up with exponentiated gradient descent (EG), i.e.
$$
(x_{t+1})_i = \frac{(x_t)_i \exp (-\eta \nabla f(x_t)_i)}{\sum_{j=1}^K (x_t)_j \exp (-\eta \nabla f(x_t)_j)}
$$
In the next section we explain how we can leverage these variations of the online mirror descent algorithm to achieve our goal of enforce monotonicity in the quantiles and how we can leverage classic dynamic regret upper bounds to control the error between the true quantiles and the estimations.

\subsection{Exponentiated Gradient Descent}

We first consider the EG update restricted to the truncated simplex $\Delta_{K+1}^{\mu} = \left\{ w \in \mathbb{R}^{K+1} \mid w_i \geq \mu, \ \sum_{i=0}^K w_i = 1 \right\}$, for some $0 < \mu < \frac{1}{K+1}$. Rather than updating the quantiles directly, we work with the normalized gaps between consecutive quantiles. Specifically, define
\[
w_{t,i} = \frac{d_{t,i}}{B} \in [0,1], \quad \text{where } d_{t,i} = q_{t,i} - q_{t,i+1},
\]
with boundary conditions $q_{t,K+1} = 0$ and $q_{t,0} = B$ for all $t \ge 1$. We initialize $w_{1,i} = \frac{1}{K+1}$ for all $i \ge 0$. This parametrization is equivalent to the change of variables
\[
q_{t,i} = B \sum_{j=i}^K w_{t,j},
\]
which can be written as $\bar{q}_t = J \bar{w}_t$, where $J_{i,j} = B\,\mathbbm{1}_{\{i \le j\}}$. Under this transformation, we can express the loss function in terms of $\bar{w}_t$ as follows:

$$
\begin{aligned}
g_t(\bar{w}) &:= f_t(J\bar{w})\\ & \text{ }= \sum_{i=1}^K \rho_t(B\sum_{j=i}^K w_{j}, 1-\alpha_i)
\end{aligned}
$$

This formulation makes explicit how each coordinate $w_i$ influences all quantiles $\{q_j\}_{j \le i}$, thereby introducing a structured coupling across coverage levels. We now apply the EG algorithm to minimize $g_t$ over the truncated simplex $\Delta_{K+1}^{\mu}$:

\begin{equation}\label{eq:updating_eg}
\tilde w_{t+1,i}
=
w_{t,i}\exp(-\eta (\nabla g_{t})_i),
\qquad
\bar{w}_{t+1}= \arg \min_{\bar{w} \in \Delta_{K+1}^{\mu}} D_{\KL}(\bar{w}, \tilde w_{t+1}),
\end{equation}

The projection step admits a simple closed form.

\begin{proposition}\label{prop:explicit_projection_eg} The KL projection onto the truncated simplex $\Delta_{K+1}^{\mu}$ is given by

$$w_{t+1, i}=\max \left\{\mu, c_t \tilde{w}_{t+1, i}\right\}$$

where

$$
c_t = \frac{1-|S_t^{c}|\mu}{\sum_{i \in S_t} \tilde{w}_{t+1, i}}, \qquad S_t := \{i: \tilde{w}_{t+1, i} \ge \mu \}
$$
\end{proposition}

Furthermore, we can calculate the gradient using the chain rule as $ \nabla g_{t} = J^{\top} \nabla f_{t} $ which yields the coordinate-wise expression

\[
(\nabla g_t)_i
=
\sum_{j=1}^K \frac{\partial \rho_t(q_{t,j},1-\alpha_j)}{\partial q_{t,j}}
\frac{\partial q_{t,j}}{\partial w_{t,i}}
=
B\sum_{j=i}^K (\alpha_j-\err_{t,j}),
\]

since $\partial_{q_{t,j}} \rho_t(q_{t,j},1-\alpha_j)=\alpha_j-\err_{t,j}$ and $\frac{\partial q_{t,j}}{\partial w_{t,i}} = B\,\mathbbm{1}_{\{i \leq j\}}$.

This reveals a key feature of the EG update: it induces a \emph{global coupling} across quantile levels. Each coordinate $w_{t,i}$ controls the gap between consecutive quantiles and influences all prediction sets up to level $i$. Correspondingly, the gradient $(\nabla g_t)_i$ aggregates the miscoverage discrepancies across these levels. When the method is overly conservative (i.e., $\err_{t,j} < \alpha_j$ on average), the cumulative term is positive and the corresponding gap is reduced; when it undercovers, the gap is increased. As a result, the EG update reallocates mass across quantile gaps in a coordinated manner, adjusting the entire quantile curve simultaneously while preserving nestedness by construction. This global interaction contrasts with projection-based approaches, where corrections are applied locally. 

Standard regret guarantees for EG, combined with mild assumptions on the distribution of $s_t$, allow us to control the deviation of the estimated quantiles from their targets.

\begin{theorem}\label{theorem:l2_quantiles_eg}
    Let $s_1, \dots, s_T$ be a sequence of random variables with densities $p_t(s) > p > 0$. Then, the quantiles $\bar{q}_t = J \bar{w}_t$ using the updating rule in Eq. \eqref{eq:updating_eg} satisfy

\begin{equation}
    \frac{1}{T} \sum_{t=1}^T \frac{p\left\|\bar{q}_t-\bar{q}_t^*\right\|_2^2}{2} \leq \frac{(1+\log(1/\mu))(1+V_T^w)}{\eta T} + (B K)^2 \frac{\eta}{2}
\end{equation}

Where $q_{t,i}^*$ is such that $\mathbb{P}\left(s_t<q_{t,i}^*\right)=1-\alpha_i$ and $V_T^{w} = \sum_{t=1}^{T-1} \|w_{t+1}^*-w_{t}^*\|_1$.

\end{theorem}

Similar to \cite{gibbs2024conformal}, if we additional assume the map $q_i \rightarrow \bbP \left( s_t \le q_i | \{s_r\}_{r<t} \right) $ is $L$-Lipschitz, then $|\bbP \left( s_t \le q_i | \{s_r\}_{r<t} \right) - (1-\alpha_i)|< L|q_{t,i} - q_{t,i}^*|$, then Theorem \ref{theorem:l2_quantiles_eg} translates to coverage error bound.

\subsection{Projected Gradient Descent}

For our second method, we will use a modified version of projected gradient descent such that we keep a minimum difference between quantiles to avoid collapsing to single points. We first update the whole quantile vector according to the gradient. Finally, we project it onto the set with decreasing elements.

\begin{equation}\label{eq:updating_proj}
    \Tilde{q}_{t+1} = \bar{q}_{t} - \eta \left(\bar{\alpha}-\overline{\err}_t\right)  \qquad \bar{q}_{t+1} = \Pi_{\cQ}(\Tilde{q}_{t+1})
\end{equation}

Where $\overline{\err}_t = (\err_{t, 1}, \dots, \err_{t, K})$ and $\cQ = \{\bar{q}: B\ge q_1\ge\dots\ge q_K\ge0\}$. The projection $\Pi_{\cQ}$ reduces to the PAVA algorithm \citep{ayer1955empirical}.

In contrast to EG, the PG method operates via local correction: it first performs independent gradient updates on each quantile and then enforces monotonicity through a projection step. This projection only adjusts neighboring quantiles as needed to restore the ordering constraint, without explicitly redistributing information across all levels. Consequently, PG ensures feasibility (nestedness) but lacks the intrinsic global coordination present in EG, which can limit its ability to exploit shared structure across coverage levels.

\begin{theorem}\label{theorem:l2_quantiles_projected}
    Let $s_1, \dots, s_T$ be a sequence of random variables with densities $p_t(s) > p > 0$. Then, the quantiles $\bar{q}_t$ from the update rule in Eq. \eqref{eq:updating_proj} satisfy

\begin{equation}
    \frac{1}{T} \sum_{t=1}^T \frac{p \left\|\bar{q}_t-\bar{q}_t^*\right\|_2^2}{2} \leq \frac{3B^2 K(1+V_T^q)}{\eta T} + \eta K 
\end{equation}

Where $q_{t,i}^*$ is such that $\mathbb{P}\left(s_t<q_{t,i}^*\right)=1-\alpha_i$ and $V_T^{q} = \sum_{t=1}^{T-1} \|q_{t+1}^*-q_{t}^*\|_1$

\end{theorem}

\section{Experiments}\label{section:exp}

In this section we evaluate our methods as well as the previously proposed quantile tracker and naive adaptations to enforce nestedness. We measure three key properties:
(i) calibration across the full risk spectrum,
(ii) nestedness of prediction sets, and
(iii) statistical efficiency of joint quantile estimation. We evaluate performance across multiple coverage levels 
$\{\alpha_i\}_{i=1}^K$ using the following metrics:

\subsection{Metrics}\label{subsection:metrics}

\textbf{Full risk spectrum calibration error.} 
This metric evaluates the deviation between the empirical miscoverage rate and the target level $\alpha_i$ across all coverage levels. While the quantile tracker is expected to perform well under this metric, it is known to exhibit overcorrection behavior \citep{angelopoulos2024online}. In particular, the algorithm may alternate between periods of undercoverage and overly conservative predictions, producing wider intervals that compensate for past errors. As a result, a low average calibration error can be misleading, as it may reflect oscillatory behavior rather than stable and well-calibrated uncertainty estimates.

$$
\mathrm{CE}_i
=
\left|
\frac{1}{T}\sum_{t=1}^T 
\mathbbm{1}\{Y_t \notin \mathcal{C}_t^{i}(X_t)\}
-
\alpha_i
\right|
$$

\textbf{Quantile estimation error.} 
In the simulated setting, the true quantiles are known, allowing us to directly compare them with the online estimates. This metric evaluates how accurately each method tracks the underlying quantile dynamics over time. In contrast to calibration error, which reflects long-run coverage, this measure provides a more direct assessment of statistical efficiency and the ability to adapt to stochastic variability in the data-generating process.

$$
\|q_t - q_t^*\|_1 = \sum_{i=1}^K |q_{t,i} - q_{t,i}^*|
$$

\textbf{Prediction set size.}
We calculate the rolling average width of $\mathcal{C}_t^i$ across $t$:, this jointly 

$$\sum_{t-w}^t|2q_s|$$

\textbf{Nestedness violations.}
Gap between consecutive quantiles $q_{t,i} - q_{t,i+1}$.
This will be positive for all $t, i$ and all methods except the standard quantile tracking which does not guarantee nestedness.


\subsection{Baselines}\label{subsection:baselines}

We compare our methods against two baselines: the classical quantile tracker and a projected quantile tracker. The first baseline applies the quantile tracker update in Eq.~\eqref{eq:update_q_multi_grad} independently across coverage levels. As discussed earlier, this update rule does not enforce monotonicity across quantiles and therefore does not guarantee nested prediction sets.

The second baseline is a simple extension of the quantile tracker that enforces monotonicity via projection. Specifically, after performing the standard update in Eq.~\eqref{eq:update_q_multi_grad}, the resulting quantiles are projected onto the set of monotone sequences. While this approach guarantees nestedness, it does not incorporate any coupling between quantile levels during the update step, and therefore does not fully exploit the benefits of joint estimation.

\begin{equation}
    \Tilde{q}_{t+1} = \Tilde{q}_{t} + \eta \left(\overline{\err}_t-\bar{\alpha}\right), \qquad \bar{q}_{t+1} = \Pi_{\cQ}(\Tilde{q}_{t+1})
\end{equation}

Note that this approach differs from our projected gradient method. Here, the quantiles are updated independently using the gradient step, and only the current iterate is projected to enforce monotonicity. While this guarantees nested prediction sets, the projection acts as a post-processing step and does not introduce coupling across quantile levels during the update. As a result, it does not benefit from the joint estimation structure that our method exploits.

\subsection{Synthetic Data: Uniform scores with reflected random-walk drift}

To evaluate the ability of our online conformal procedures to track time-varying quantiles, we consider a synthetic setting in which the score distribution evolves according to a latent random walk with reflecting boundaries.

Let $(z_t)_{t\ge1}$ be a latent process taking values in a bounded interval $[a,b]$, initialized at $z_1 \in [a,b]$, and evolving as
\begin{equation}
\tilde z_{t+1} = z_t + \sigma \varepsilon_t,
\qquad
\varepsilon_t \stackrel{\mathrm{iid}}{\sim} \mathcal N(0,1),
\end{equation}
followed by reflection onto $[a,b]$:

\begin{equation}
z_{t+1} =
\begin{cases}
2a - \tilde z_{t+1}, & \tilde z_{t+1} < a,\\[4pt]
2b - \tilde z_{t+1}, & \tilde z_{t+1} > b,\\[4pt]
\tilde z_{t+1}, & \text{otherwise}.
\end{cases}
\end{equation}

\begin{figure}[htbp]
    \centering
    \begin{subfigure}[t]{0.5\textwidth}
        \centering
        \includegraphics[width=\linewidth]{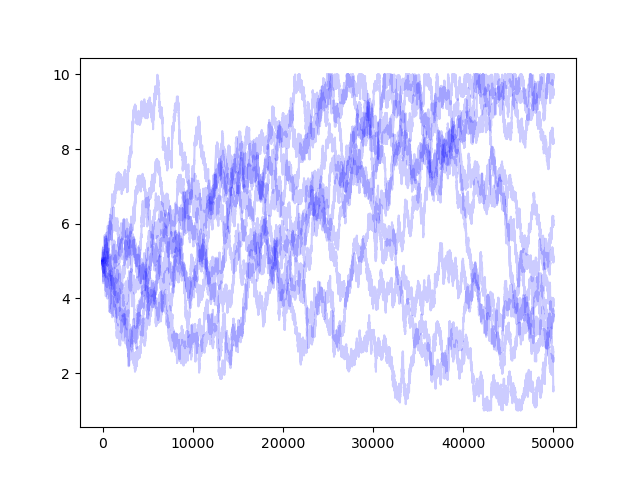}
    \end{subfigure}
    
    \caption{Simulated trajectory of the latent process $(z_t)$. The process remains confined to the interval while exhibiting gradual drift, inducing a time-varying score distribution in the synthetic experiment.} 
    \label{fig:z_paths}
\end{figure}

If multiple reflections are required (i.e., when $|\tilde z_{t+1}-z_t|$ is large), the reflection step is applied iteratively until $z_{t+1} \in [a,b]$. This construction yields a random walk that behaves locally like a Gaussian walk but is confined to a compact domain. Example simulated trajectories of $(z_t)$ are shown in Figure~\ref{fig:z_paths}. Conditional on $z_t$, the conformity score is generated as

\begin{equation}
s_t \mid z_t \sim \mathrm{Unif}\!\left[z_t-\frac{w}{2},\, z_t+\frac{w}{2}\right],
\end{equation}

where $w>0$ controls the width of the score distribution. Thus, the score distribution has fixed spread but a time-varying center governed by the reflected random walk which induces a non-stationary environment. Hence, for any target miscoverage level $\alpha \in (0,1)$, the ground truth $(1-\alpha)$-quantile is available in closed form:
\begin{equation}
q_t^*(1-\alpha)
=
z_t+\frac{w}{2}-\alpha w.
\end{equation}
For a collection of levels $\alpha_1 < \cdots < \alpha_K$, the true quantiles are
\begin{equation}
q_{t,i}^*
=
z_t+\frac{w}{2}-\alpha_i w,
\qquad i=1,\dots,K.
\end{equation}

This setting provides a controlled form of nonstationarity. The parameter $\sigma$ governs the local variability of the latent process, while the reflecting boundaries ensure that $z_t$ remains in $[a,b]$, preventing unbounded drift. As a result, the score distribution continuously evolves over time but within a fixed range, yielding a stable yet nonstationary benchmark.

A key advantage of a synthetic data experiment is that the true quantiles are known exactly, allowing us to directly assess the tracking performance of the online estimates. In particular, we evaluate the rolling average tracking error

\begin{equation}
\frac{1}{t-dt}\sum_{s=t-dt}^t \|\bar q_s-\bar q_s^*\|_1,
\end{equation}

Calibration is assessed through empirical coverage at each level using the indicators $\mathbbm{1}\{s_t \le q_{t,i}\}$.

We fix the reflecting boundaries to $a=0.5$ and $b=9.5$, and initialize the latent process at $z_1=5$. The random-walk scale is set to $\sigma=0.025$. The score distribution has fixed width $w=1$. We consider $K=9$ coverage levels given by
\[
\alpha_i = 0.1\, i, \qquad i=1,\dots,9,
\]
corresponding to target coverages $1-\alpha_i \in \{0.9,0.8,\dots,0.1\}$. The simulation is run over $t \in \{1,\dots,50{,}000\}$, and performance metrics are computed using evaluation windows of size $dt = 10{,}000$. These choices produce a moderate level of drift while keeping the score distribution well within the interval $[0, 10]$, ensuring stable behavior throughout the simulation.

Figure~\ref{fig:l1_dist} shows that the EG update achieves a lower $\ell_1$ tracking error, supporting the hypothesis that sharing information across quantiles improves statistical efficiency. In contrast, the projected gradient method closely mirrors the behavior of the independent quantile tracker, indicating limited gains from joint optimization in this setting.

\begin{figure}[htbp]
    \centering
    \begin{subfigure}[t]{0.6\textwidth}
        \centering
        \includegraphics[width=\linewidth]{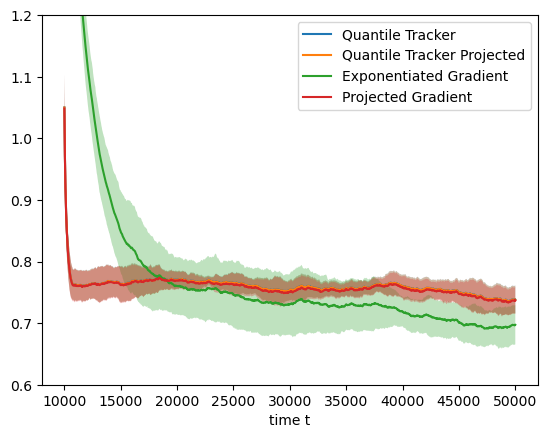}
    \end{subfigure}
    
    \caption{Rolling average of the $\ell_1$ distance between $q_t$ and $q_t^*$} 
    \label{fig:l1_dist}
\end{figure}

\subsection{Real Data: US Inflation}

We evaluate our methods on the US Consumer Price Index (CPI) for All Urban Consumers, published by the Bureau of Labor Statistics and distributed via FRED \citep{bls_cpiaucsl}. Inflation forecasting is a natural application for our framework, as policymakers and market participants often require uncertainty estimates at multiple confidence levels to support decisions under heterogeneous risk preferences.

We use monthly CPI data from 1950 to 2025 and construct the yearly inflation series as
\[
y_t = \frac{\text{CPI}_t - \text{CPI}_{t-12}}{\text{CPI}_{t-12}}.
\]
To generate predictions, we employ an autoregressive model of order three, $\mathrm{AR}(3)$, which captures short-term temporal dependencies in the inflation series. At each time $t$, we fit the model
\[
y_t = \beta_0 + \beta_1 y_{t-1} + \beta_2 y_{t-2} + \beta_3 y_{t-3} + \varepsilon_t,
\]
using a rolling window of the previous five years of monthly observations (i.e., 60 data points). The model is retrained sequentially at each time step, and the fitted model is used to produce a one-step-ahead prediction $\hat{y}_{t+1}$. This rolling-window procedure allows the predictive model to adapt to evolving dynamics in the inflation process, reflecting its non-stationary nature.

We use as the nonconformity score the absolute prediction error,
\[
s_t = |y_t - \hat{y}_t|,
\]
which measures the deviation between observed and predicted inflation. These scores are then used to construct prediction sets of the form
\[
\mathcal{C}_t^{i} = \{ y \in \mathbb{R} : |y - \hat{y}_t| \leq q_{t,i} \},
\]
where the thresholds $q_{t,i}$ are updated online using the proposed methods corresponding to the coverage level $1-\alpha_i$ for $i \in [K]$. We consider $K=99$ coverage levels given by
\[
\alpha_i = 0.01\, i, \qquad i=1,\dots,99,
\]
corresponding to target coverages $1-\alpha_i \in \{0.99,0.8,\dots,0.01\}$.

\begin{figure}[htbp]
    \centering
    \begin{subfigure}[t]{0.6\textwidth}
        \centering
        \includegraphics[width=\linewidth]{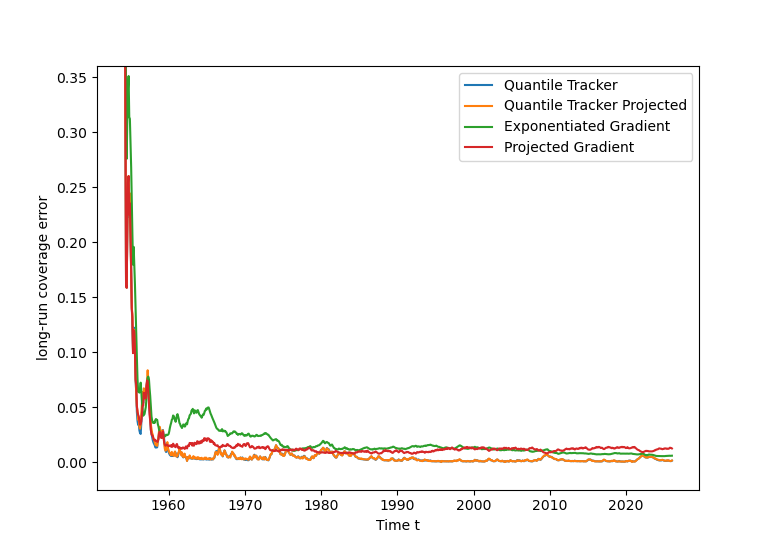}
    \end{subfigure}
    
    \caption{Cumulative average of the sum of mis coverage errors $\text{CE}_i$} 
    \label{fig:cov_error_inflation}
\end{figure}

Figure~\ref{fig:cov_error_inflation} reports the cumulative average calibration error aggregated across all miscoverage levels $\{\alpha_i\}_{i=1}^K$. 

$$\sum_{i=1}^K \left|
\frac{1}{T}\sum_{t=1}^T 
\mathbbm{1}\{Y_t \notin \mathcal{C}_t^{i}(X_t)\}
-
\alpha_i
\right|$$

Specifically, at each time $t$ we compute the average deviation between the empirical miscoverage and the target levels, and report its cumulative average over time. This metric captures how well each method calibrates the entire risk spectrum, rather than a single coverage level. We observe that all methods eventually achieve small calibration error, indicating convergence to the desired long-run coverage. However, their transient behaviors differ. The unconstrained quantile tracker converges more rapidly toward zero, which is consistent with its tendency to overcorrect by alternating between undercoverage and overly conservative intervals. In contrast, the proposed methods exhibit slightly slower but more stable trajectories. This suggests that enforcing structural constraints to preserve nestedness introduces a controlled trade-off, leading to more stable calibration dynamics while maintaining accurate long-run coverage.

\begin{figure}[htbp]
    \centering
    
    \begin{subfigure}{0.4\textwidth}
        \includegraphics[width=\linewidth]{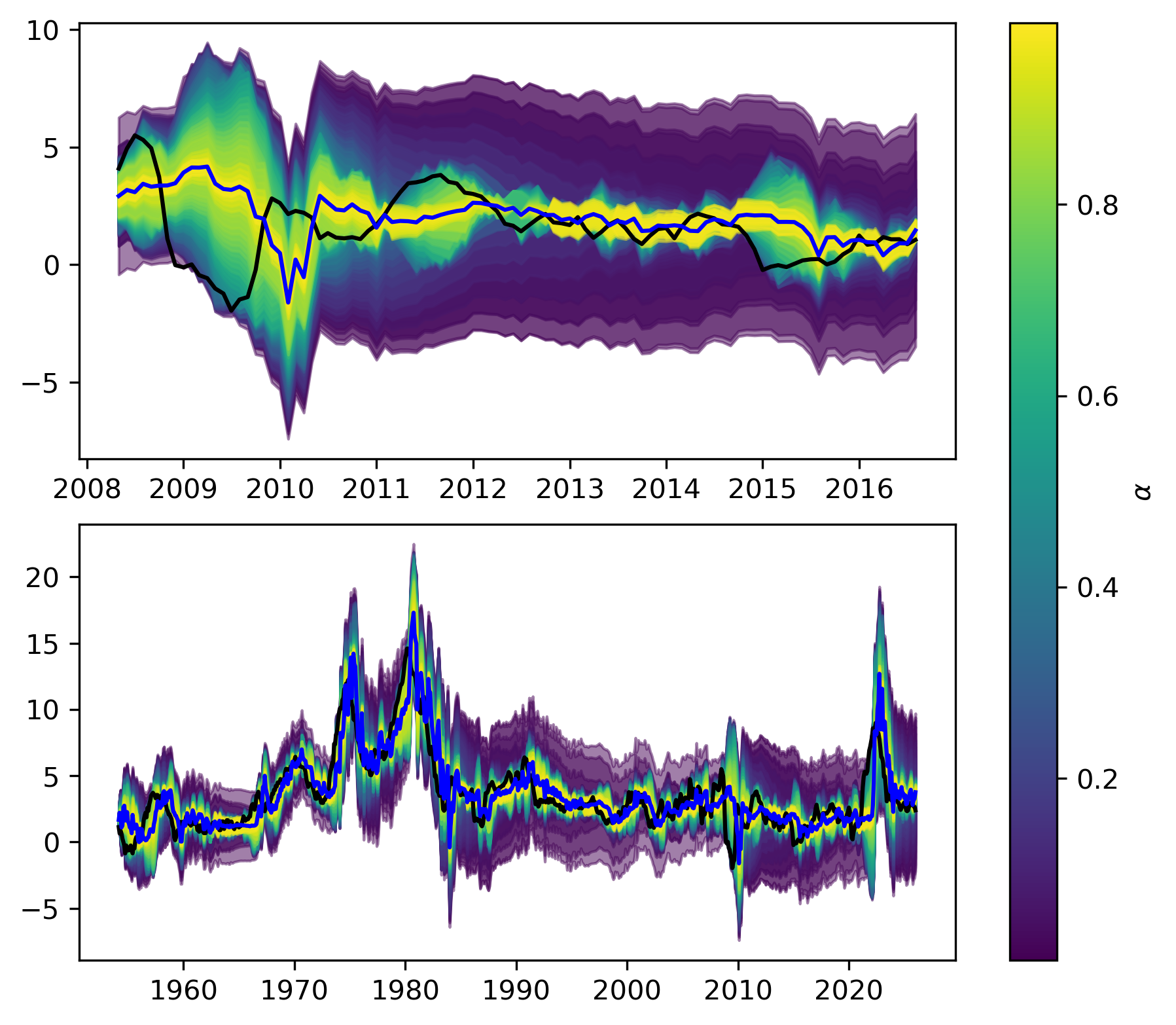}
        \caption{Quantile Tracker}
    \end{subfigure}
    \begin{subfigure}{0.4\textwidth}
        \includegraphics[width=\linewidth]{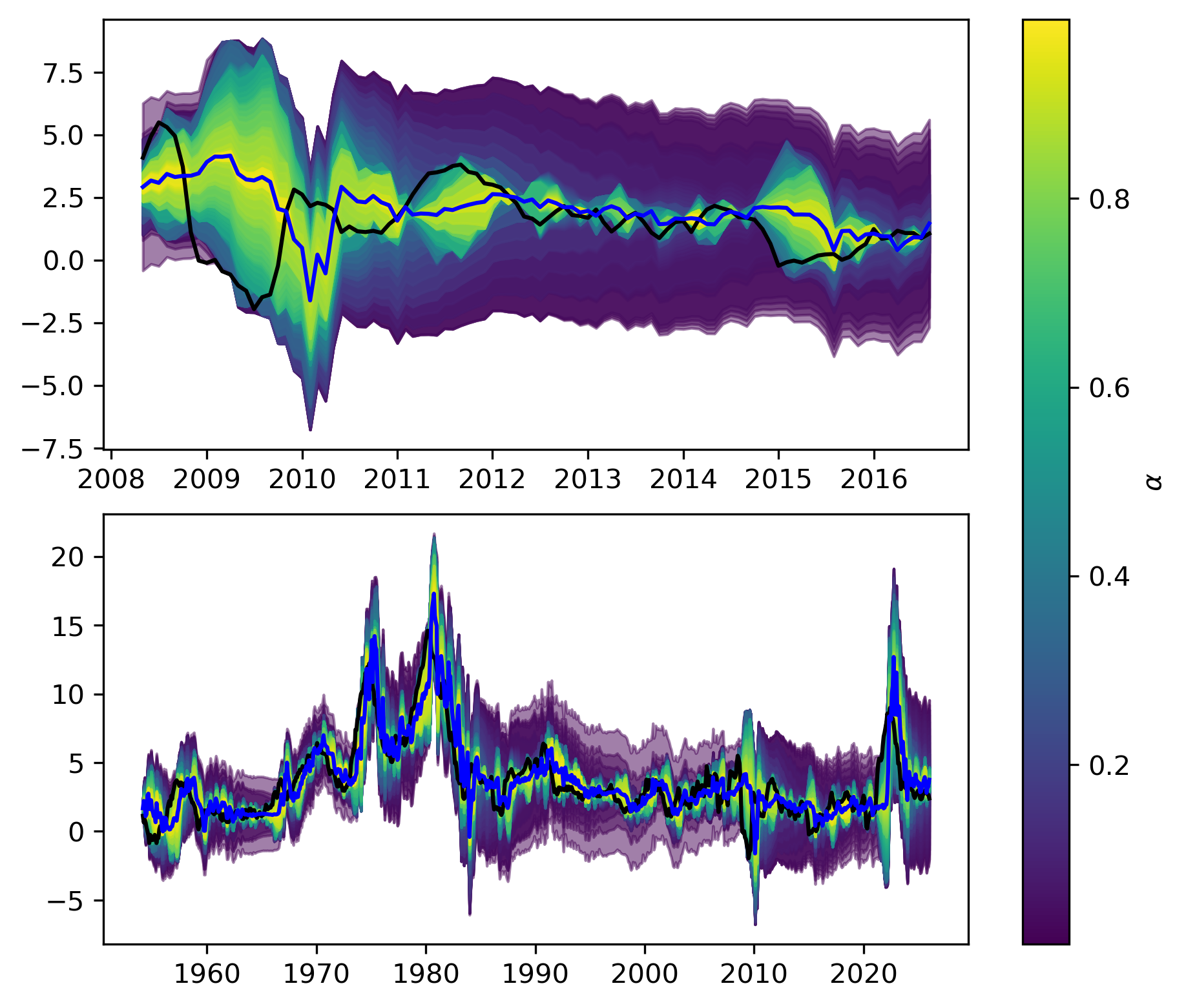}
        \caption{Quantile Tracker Projected}
    \end{subfigure}
    \vspace{.5cm}
    
    \begin{subfigure}{0.4\textwidth}
        \includegraphics[width=\linewidth]{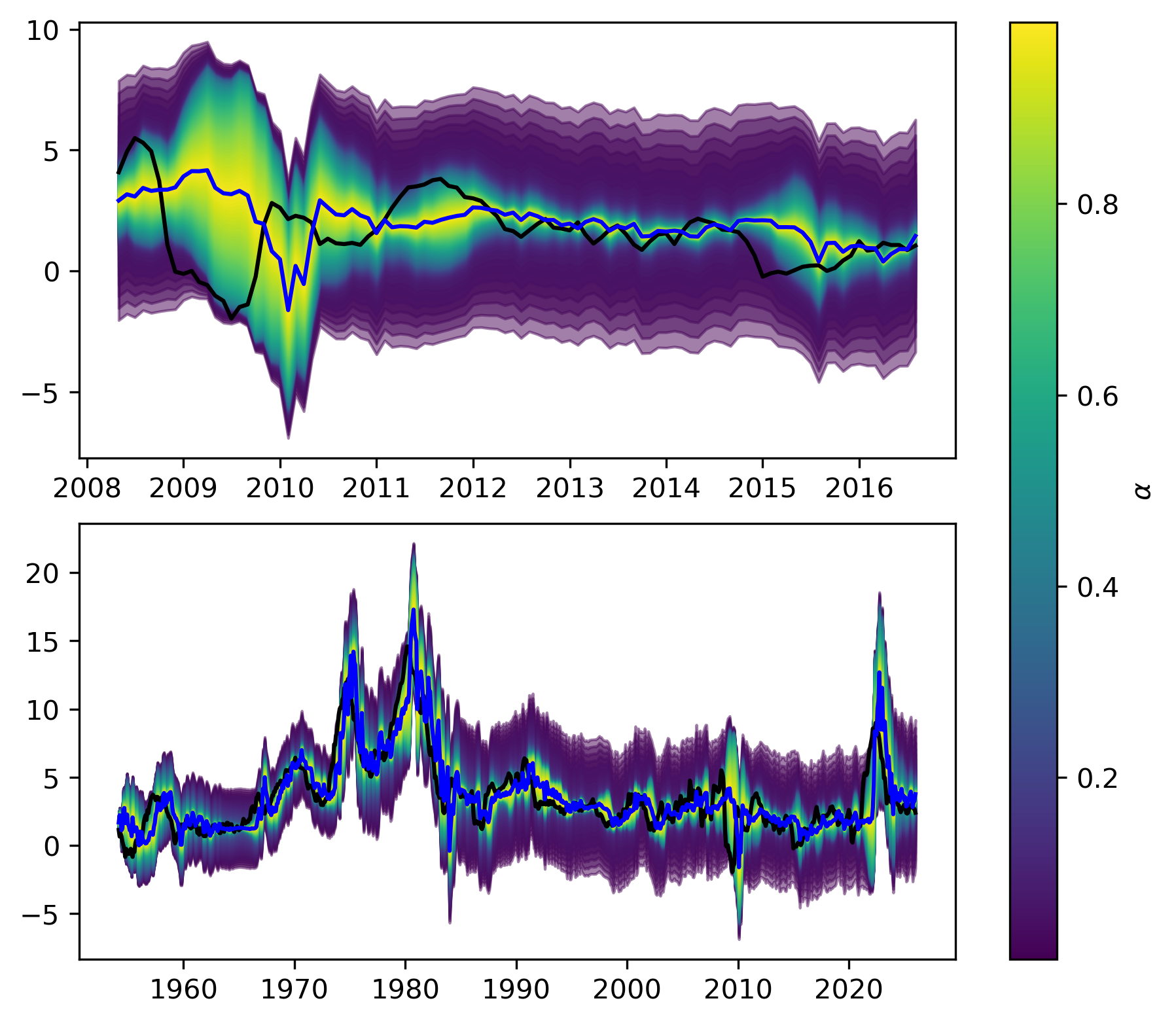}
        \caption{Exponentiated Gradient}
    \end{subfigure}
    \begin{subfigure}{0.4\textwidth}
        \includegraphics[width=\linewidth]{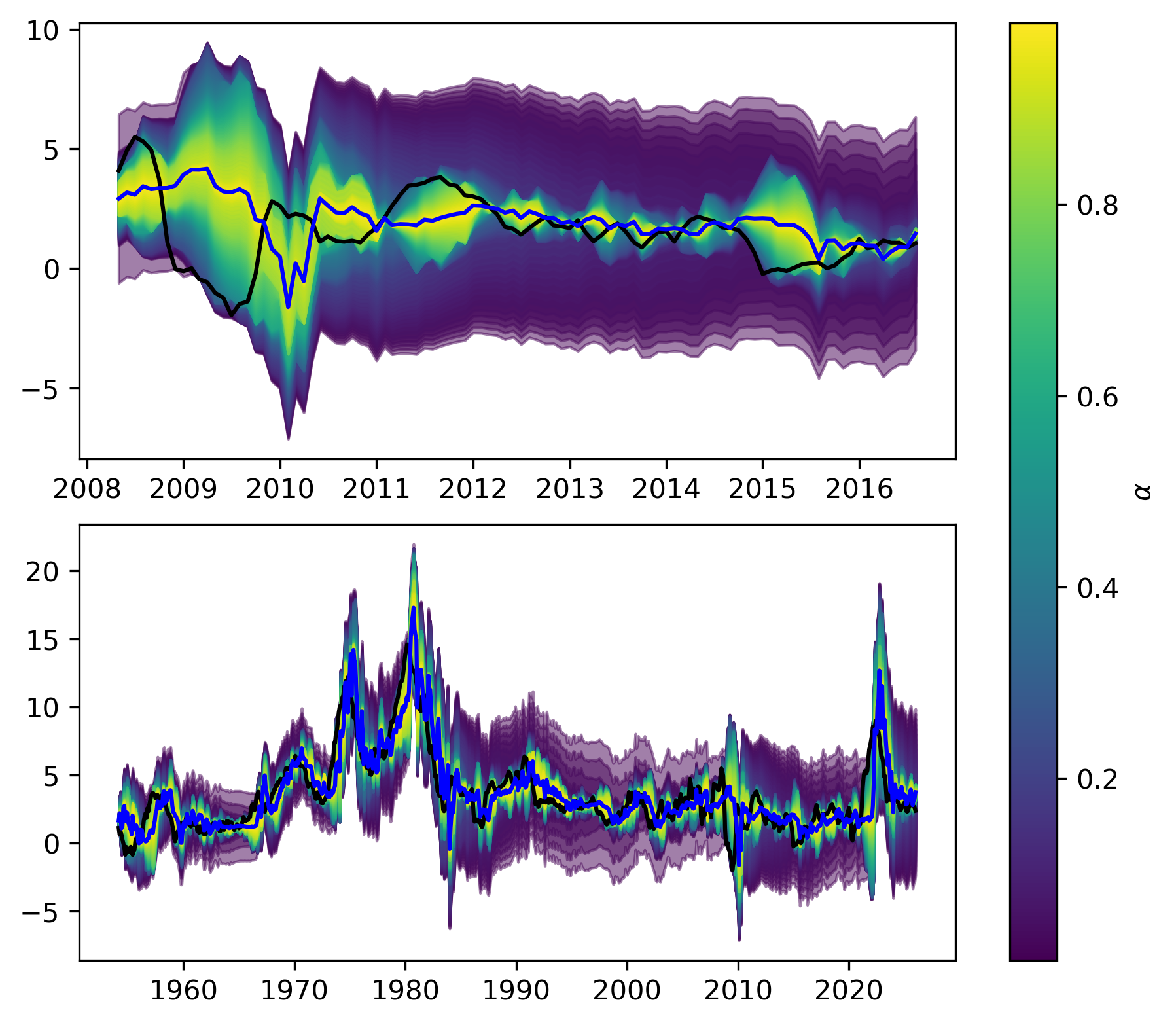}
        \caption{Projected Gradient}
        
    \end{subfigure}
    \caption{Prediction intervals across multiple coverage levels in the whole window (bottom panels) and zoom after the 2008 crises (top panels): (a) standard quantile tracker method, (b) quantile tracker with the last iteration projected, (c) exponentiated gradient and (d) projected gradient}
    \label{fig:prediction_sets}
\end{figure}

Figure~\ref{fig:prediction_sets} visualizes the resulting prediction sets across multiple coverage levels. Each color corresponds to a different miscoverage level $\alpha$, forming a fan chart representation of the estimated quantile function over time. All methods adapt to changes in uncertainty, with prediction sets widening during periods of high volatility (e.g., around 1980, 2010, and 2020) and contracting during more stable regimes.

However, important qualitative differences emerge across methods. Baseline approaches (top panels) exhibit higher variability and occasional distortions in the quantile structure, leading to irregular and sometimes non-nested prediction bands. In contrast, the proposed constrained methods (bottom panels) produce smoother and more coherent prediction sets, consistently preserving the nested structure across all coverage levels. This behavior highlights the benefit of enforcing monotonicity jointly with online updates: beyond guaranteeing structural consistency, it stabilizes the learning dynamics and yields more interpretable and reliable uncertainty estimates in practice. In addition, The EG update yields smoother prediction sets by coordinating adjacent quantiles through cumulative discrepancies, reducing abrupt changes across levels.

\section{Discussion}

This work is motivated by practical settings in which uncertainty must be assessed across multiple risk levels simultaneously. In many applications---including forecasting, risk management, and scientific decision-making---different stakeholders operate under heterogeneous risk tolerances and require prediction sets at several coverage levels. Standard online conformal methods address each level independently, which can lead to incoherent and non-nested prediction sets. By jointly estimating multiple quantiles under structural constraints, our approach provides a coherent \emph{risk spectrum} with valid coverage guarantees and strictly nested prediction sets, improving both interpretability and usability in downstream decision-making.

A key insight of our methods lies in how information is shared across coverage levels. The exponentiated gradient (EG) update induces a form of \emph{global coupling}, where each update aggregates miscoverage discrepancies across multiple levels. This leads to coordinated adjustments of the entire quantile curve, allowing the method to borrow strength across levels and improving statistical efficiency. In contrast, the projected gradient (PG) update operates through \emph{local corrections}: quantiles are first updated independently and then minimally adjusted via projection to restore monotonicity. While this guarantees feasibility, it lacks the intrinsic global coordination of EG. Empirically, this distinction manifests in improved quantile estimation accuracy for EG, particularly in non-stationary environments where sharing information across levels is beneficial.

Beyond accuracy, the global structure of the EG update also impacts the \emph{smoothness} of the estimated quantile function across coverage levels. Because updates are driven by cumulative discrepancies, adjacent quantiles evolve in a coordinated manner, reducing irregularities and abrupt changes across levels. This leads to smoother and more stable prediction bands, which is particularly desirable in applications such as forecasting, where interpretability of uncertainty is critical. In contrast, methods based on independent updates or post-hoc projection may exhibit more variability and local distortions in the quantile structure.

An important direction for future work is the incorporation of adaptive or data-driven learning rates. Recent advances in online conformal prediction have shown that carefully designed step-size sequences can improve adaptivity to distributional shifts and provide stronger guarantees across stochastic and adversarial regimes. Integrating such adaptive mechanisms into the joint quantile estimation framework may further enhance performance, allowing the methods to respond more effectively to non-stationarity while preserving the structural constraints imposed here.

Finally, our work contributes to a broader understanding of online conformal prediction beyond worst-case guarantees. While many existing methods focus on adversarial settings, practical applications often exhibit stochastic structure with evolving distributions. By framing the problem as one of tracking time-varying quantiles and analyzing performance through dynamic regret, our approach helps bridge the gap between adversarial robustness and statistical efficiency. This perspective opens the door to a richer analysis of online conformal methods in stochastic environments, where the goal is not only to guarantee coverage, but also to accurately and efficiently track the underlying uncertainty over time.

\newpage

\bibliographystyle{achemso}
\bibliography{references}

\newpage
\begin{alphasection}

\section{Proofs}

\textbf{Proof Proposition \ref{prop:explicit_projection_eg}}
    \begin{proof}
Since $\tilde w_i>0$ for all $i$, the map
\[
w\mapsto D_{\mathrm{KL}}(w\|\tilde w)=\sum_{i=1}^n w_i\log\frac{w_i}{\tilde w_i}
\]
is strictly convex on $\mathbb{R}_+^n$. As $\mathcal X$ is a nonempty closed convex set, the minimizer exists and is unique.

Consider the constrained problem
\[
\min_{w\in\mathbb{R}^n}\ \sum_{i=1}^n w_i\log\frac{w_i}{\tilde w_i}
\qquad
\text{s.t.}\qquad
w_i\ge \mu,\quad \sum_{i=1}^n w_i=1.
\]
Its Lagrangian is
\[
\mathcal L(w,\lambda,\nu)
=
\sum_{i=1}^n w_i\log\frac{w_i}{\tilde w_i}
+\lambda\Big(\sum_{i=1}^n w_i-1\Big)
+\sum_{i=1}^n \nu_i(\mu-w_i),
\]
where $\nu_i\ge 0$. Let $w^*$ be the optimizer. By the KKT conditions,
\[
\log\frac{w_i^*}{\tilde w_i}+1+\lambda-\nu_i=0,
\qquad i=1,\dots,n.
\]
Hence
\[
w_i^*=\tilde w_i e^{-1-\lambda+\nu_i}.
\]

Now, if $w_i^*>\mu$, then complementary slackness gives $\nu_i=0$, and therefore
\[
w_i^*=c\,\tilde w_i,
\qquad\text{where } c:=e^{-1-\lambda}>0.
\]
If instead $w_i^*=\mu$, then the lower-bound constraint is active. Thus, for every $i$,
\[
w_i^*=\max\{\mu,c\,\tilde w_i\}.
\]
Finally, the equality constraint $\sum_{i=1}^n w_i^*=1$ implies that $c>0$ must satisfy
\[
\sum_{i=1}^n \max\{\mu,c\,\tilde w_i\}=1.
\]
Since the left-hand side is continuous and strictly increasing in $c$, such a $c$ exists and is unique. This proves the claim.
\end{proof}

\begin{proposition}\label{prop:lower_bound_regret}
    (Proposition 5 in \cite{gibbs2024conformal}) Let $s$ be a random variable and assume that there exists a value $q^*$ such that $\mathbb{P}\left(s<q^*\right)=\alpha$. Then, for any $q$,

$$
\mathbb{E}\left[\rho_t(q, 1-\alpha)\right]-\mathbb{E}\left[\rho_t(q^*, 1-\alpha)\right] = 
\begin{cases} 
\mathbb{E}\left[(q-s) \mathbbm{1}_{q < s \leq q^*}\right], \text { if } q \leq q^*  \\ 
\mathbb{E}\left[(s-q) \mathbbm{1}_{q^* < s \leq q}\right], \text { if } q^* \leq q
\end{cases}
$$

So, in particular, if $s$ has a density $p(\cdot)$ on $[0,B]$ with $p(x) \geq p>0$ for all $x \in[0,B]$, then

$$
\mathbb{E}\left[\rho_t(q, 1-\alpha)\right]-\mathbb{E}\left[\rho_t(q^*, 1-\alpha)\right]\geq \frac{p\left(q-q^*\right)^2}{2}
$$
\end{proposition}

\begin{lemma}\label{lemma:dyn_regret_eg}
    Let $s_1, \dots, s_T$ be an arbitrary sequence in $[0, B]$, then the EG update rule has the following regret bound.

\begin{equation}
    \sum_{t=1}^T g_t\left(w_t\right)-\sum_{t=1}^T g_t\left(w_t^{*}\right) 
\le
\frac{(1+\log(1/\mu))}{\eta} \left(1+\sum_{t=1}^{T-1}\norm{w_{t+1}^*-w_t^*}_1\right)
+
\eta G_\infty^2 T
\end{equation} 

\end{lemma}

\begin{proof}
Since the projected exponentiated gradient update is exactly the mirror descent update with negative entropy, we may write
\[
w_{t+1}
=
\arg\min_{w \in \Delta_{K+1}^\mu}
\left\{
\eta \langle \nabla g_t, w \rangle + D_{\mathrm{KL}}(w \| w_t)
\right\}.
\]
We first prove the standard one-step mirror descent inequality.

\medskip

\noindent\textbf{Step 1: One-step inequality.}
Fix $t$ and let $w_t^* \in \Delta_{K+1}^\mu$

$$
\begin{aligned}
 \eta \langle \nabla g_t, w_{t}-w_t^* \rangle
&= \left\langle
\nabla \phi(w_{t}) - \nabla \phi(w_{t+1}) -\eta \nabla g_t,
\, w_t^* - w_{t+1}
\right\rangle \\
&+ \left\langle
\nabla \phi(w_{t+1}) - \nabla \phi(w_{t}),
\, w_t^* - w_{t+1}
\right\rangle \\
&+\left\langle \eta \nabla g_t,
\, w_t - w_{t+1}
\right\rangle
\end{aligned}
$$

By optimality of $w_{t+1}$ over the convex set $\Delta_{K+1}^\mu$, 
\[
\left\langle
\eta \nabla g_t + \nabla \phi(w_{t+1}) - \nabla \phi(w_t),
\, w_t^* - w_{t+1}
\right\rangle
\ge 0.
\]

Therefore,
\[
\eta \langle \nabla g_t, w_{t}-w_t^* \rangle
\le \left\langle \eta \nabla g_t,
\, w_t - w_{t+1}
\right\rangle + \left\langle
\nabla \phi(w_{t+1}) - \nabla \phi(w_{t}),
\, w_t^* - w_{t+1}
\right\rangle
\]

Using the three-point identity for Bregman divergences,
\[
\left\langle
\nabla \phi(w_{t+1}) - \nabla \phi(w_t),
\, w_t^* - w_{t+1}
\right\rangle
=
D_{\mathrm{KL}}(w_t^* \| w_t)
-
D_{\mathrm{KL}}(w_t^* \| w_{t+1})
-
D_{\mathrm{KL}}(w_{t+1} \| w_t),
\]
Hence
\[
\eta \langle \nabla g_t, w_t-w_t^* \rangle
\le
\eta \langle \nabla g_t, w_t-w_{t+1} \rangle
+
D_{\mathrm{KL}}(w_t^* \| w_t)
-
D_{\mathrm{KL}}(w_t^* \| w_{t+1})
-
D_{\mathrm{KL}}(w_{t+1} \| w_t).
\]
Using H\"older's inequality,

$$
\begin{aligned}
\eta \langle \nabla g_t, w_t-w_t^* \rangle
&\le
\eta \langle \nabla g_t, w_t-w_{t+1} \rangle
+
D_{\mathrm{KL}}(w_t^* \| w_t)
-
D_{\mathrm{KL}}(w_t^* \| w_{t+1})
-
D_{\mathrm{KL}}(w_{t+1} \| w_t)\\
&\le \eta \|\nabla g_t\|_\infty \|w_t-w_{t+1}\|_1 + D_{\mathrm{KL}}(w_t^* \| w_t)
-
D_{\mathrm{KL}}(w_t^* \| w_{t+1})
-
D_{\mathrm{KL}}(w_{t+1} \| w_t)\\
&\le \frac{\eta^2}{2}\|\nabla g_t\|_\infty^2
+
\frac{1}{2}\|w_t-w_{t+1}\|_1^2 + D_{\mathrm{KL}}(w_t^* \| w_t)
-
D_{\mathrm{KL}}(w_t^* \| w_{t+1})
-
D_{\mathrm{KL}}(w_{t+1} \| w_t)\\
&\le \frac{\eta^2}{2}\|\nabla g_t\|_\infty^2 + D_{\mathrm{KL}}(w_t^* \| w_t)
-
D_{\mathrm{KL}}(w_t^* \| w_{t+1})
\end{aligned}
$$

Since $\|\nabla g_t\|_\infty \le G_{\infty}$,
\[
\eta \langle \nabla g_t, w_t-w_t^* \rangle
\le
D_{\mathrm{KL}}(w_t^* \| w_t)
-
D_{\mathrm{KL}}(w_t^* \| w_{t+1})
+
\frac{\eta^2 G_{\infty}^2}{2}.
\]

\medskip

\noindent\textbf{Step 2: Sum over time.}
Summing over $t=1,\dots,T$ yields
\[
\eta \sum_{t=1}^T \langle \nabla g_t, w_t-w_t^* \rangle
\le
\sum_{t=1}^T
\Bigl(
D_{\mathrm{KL}}(w_t^* \| w_t)
-
D_{\mathrm{KL}}(w_t^* \| w_{t+1})
\Bigr)
+
\frac{\eta^2 G_{\infty}^2 T}{2}.
\]
We rewrite the KL terms as

$$
\begin{aligned}
\sum_{t=1}^T
\Bigl(
D_{\mathrm{KL}}(w_t^* \| w_t)
-
D_{\mathrm{KL}}(w_t^* \| w_{t+1})
\Bigr)
&= \sum_{t=1}^T
\Bigl(
D_{\mathrm{KL}}(w_t^* \| w_t)
- D_{\mathrm{KL}}(w_{t+1}^* \| w_{t+1}) \\
&+
D_{\mathrm{KL}}(w_{t+1}^* \| w_{t+1}) -
D_{\mathrm{KL}}(w_t^* \| w_{t+1})
\Bigr) \\
&= D_{\mathrm{KL}}(w_1^* \| w_1)
-
D_{\mathrm{KL}}(w_{T+1}^* \| w_{T+1})\\
&+
\sum_{t=1}^{T-1}
\Bigl(
D_{\mathrm{KL}}(w_{t+1}^* \| w_{t+1})
-
D_{\mathrm{KL}}(w_t^* \| w_{t+1})
\Bigr).
\end{aligned}
$$
Since KL divergence is nonnegative,
\[
-D_{\mathrm{KL}}(w_{T+1}^* \| w_{T+1}) \le 0,
\]
so
\[
\eta \sum_{t=1}^T \langle \nabla g_t, w_t-w_t^* \rangle
\le
D_{\mathrm{KL}}(w_1^* \| w_1)
+
\sum_{t=1}^{T-1}
\Bigl(
D_{\mathrm{KL}}(w_{t+1}^* \| w_{t+1})
-
D_{\mathrm{KL}}(w_t^* \| w_{t+1})
\Bigr)
+
\frac{\eta^2 G_{\infty}^2 T}{2}.
\]

\medskip

\noindent\textbf{Step 3: KL is Lipschitz in the first argument on $\Delta_{K+1}^\mu$.}
We have
\[ 
D_{\mathrm{KL}}(w_{t+1}^* \| w_{t+1})
-
D_{\mathrm{KL}}(w_t^* \| w_{t+1})
\le
(1+\log(1/\mu))\|w_{t+1}^*-w_t^*\|_1.
\]
Therefore,
\[
\sum_{t=1}^{T-1}
\Bigl(
D_{\mathrm{KL}}(w_{t+1}^* \| w_{t+1})
-
D_{\mathrm{KL}}(w_t^* \| w_{t+1})
\Bigr)
\le
(1+\log(1/\mu))
\sum_{t=1}^{T-1}\|w_{t+1}^*-w_t^*\|_1.
\]

\medskip

\noindent\textbf{Step 4: Conclude the regret bound.}
Combining the previous inequalities,
\[
\eta \sum_{t=1}^T \langle \nabla g_t, w_t-w_t^* \rangle
\le
D_{\mathrm{KL}}(w_1^* \| w_1)
+
(1+\log(1/\mu))
\sum_{t=1}^{T-1}\|w_{t+1}^*-w_t^*\|_1
+
\frac{\eta^2 G_{\infty}^2 T}{2}.
\]
Dividing by $\eta$ gives
\[
\sum_{t=1}^T \langle \nabla g_t, w_t-w_t^* \rangle
\le
\frac{D_{\mathrm{KL}}(w_1^* \| w_1)}{\eta}
+
\frac{1+\log(1/\mu)}{\eta}
\sum_{t=1}^{T-1}\|w_{t+1}^*-w_t^*\|_1
+
\frac{\eta G_{\infty}^2 T}{2}.
\]
Finally, since $\nabla g_t \in \partial g_t(w_t)$ and $g_t$ is convex,
\[
g_t(w_t)-g_t(w_t^*)
\le
\langle \nabla g_t, w_t-w_t^* \rangle.
\]
Then
\[
\sum_{t=1}^T \bigl(g_t(w_t)-g_t(w_t^*)\bigr)
\le
\frac{D_{\mathrm{KL}}(w_1^* \| w_1)}{\eta}
+
\frac{1+\log(1/\mu)}{\eta}
\sum_{t=1}^{T-1}\|w_{t+1}^*-w_t^*\|_1
+
\frac{\eta G_{\infty}^2 T}{2}.
\]

To bound the initial divergence, note that for any $u,w \in \Delta_{K+1}^\mu$,
\[
D_{\mathrm{KL}}(u\|w)
=
\sum_{i=1}^N u_i \log\frac{u_i}{w_i}
\le
\sum_{i=1}^N u_i \log\frac{1}{\mu}
=
\log(1/\mu).
\]
Thus,
\[
D_{\mathrm{KL}}(w_1^* \| w_1) \le 1+\log(1/\mu),
\]
which gives the simplified bound.
\end{proof}

\begin{lemma}\label{lemma:dyn_regret_proj} (Theorem 10.1. in \cite{hazan2016introduction})
    Let $s_1, \dots, s_T$ be an arbitrary sequence in $[0, B]$, then the projected gradient on $\cQ$ update rule has the following regret bound.

\begin{equation}
    \sum_{t=1}^T f_t\left(q_t\right)-\sum_{t=1}^T f_t\left(q_t^{*}\right) \le \frac{3D^2}{\eta}\left(1+\sum_{t=1}^{T-1}\norm{q_{t+1}^*-q_t^*}_1\right) + \eta G_{2}^2T
\end{equation} 

Where $D$ is the diameter of $\cQ$ and $\norm{\nabla f}_{2} \le G_{2}$.
\end{lemma}

\textbf{Proof Theorem \ref{theorem:l2_quantiles_eg}}

\begin{proof}
    We can apply Proposition \ref{prop:lower_bound_regret} and Lemma \ref{lemma:dyn_regret_eg} so 

$$
\begin{aligned}
\sum_{t=1}^T\frac{p\norm{q_t-q_t^*}_2^2}{2} &= \sum_{t=1}^T\sum_{i=1}^K \frac{p\left(q_{t, i}-q_{t, i}^*\right)^2}{2} \\
&\le \sum_{t=1}^T\sum_{i=1}^{K} \mathbb{E}\left[\rho_t(q_{t, i}, 1-\alpha_i)\right]-\mathbb{E}\left[\rho_t(q_{t, i}^*, 1-\alpha_i)\right] \\
&= \mathbb{E}\left[\sum_{t=1}^T \sum_{i=1}^{K} \left(\rho_t(q_{t, i}, 1-\alpha_i)-\rho_t(q_{t, i}^*, 1-\alpha_i)\right)\right] \\
&=  \mathbb{E}\left[\sum_{t=1}^Tg_t(w_t) - g_t(w_t^*)\right] \\ 
&\le \frac{(1+\log(1/\mu))}{\eta} \left(1+\sum_{t=1}^{T-1}\norm{w_{t+1}^*-w_t^*}_1\right)
+
\eta (B K)^2 T
\end{aligned}
$$
  
\end{proof}

\textbf{Proof Theorem \ref{theorem:l2_quantiles_projected}}

\begin{proof}
    We can apply Proposition \ref{prop:lower_bound_regret} and Lemma \ref{lemma:dyn_regret_proj} so 

$$
\begin{aligned}
\sum_{t=1}^T\frac{p\norm{q_t-q_t^*}_2^2}{2} &= \sum_{t=1}^T\sum_{i=1}^K \frac{p\left(q_{t, i}-q_{t, i}^*\right)^2}{2} \\
&\le \sum_{t=1}^T\sum_{i=1}^{K} \mathbb{E}\left[\rho_t(q_{t, i}, 1-\alpha_i)\right]-\mathbb{E}\left[\rho_t(q_{t, i}^*, 1-\alpha_i)\right] \\
&= \mathbb{E}\left[\sum_{t=1}^T \sum_{i=1}^{K} \left(\rho_t(q_{t, i}, 1-\alpha_i)-\rho_t(q_{t, i}^*, 1-\alpha_i)\right)\right] \\
&=  \mathbb{E}\left[\sum_{t=1}^Tf_t(q_t) - f_t(q_t^*)\right] \\ 
&\le \frac{3B^2K}{\eta}\left(1+\sum_{t=1}^{T-1}\norm{q_{t+1}^*-q_t^*}_1\right) + \eta KT
\end{aligned}
$$
\end{proof}

\end{alphasection}

\end{document}

%% file: macro.tex
\usepackage{dsfont}
\usepackage{natbib}
\usepackage{amssymb}
\usepackage{amsmath,amsthm,amssymb,amsfonts}
\usepackage{algorithm}
\usepackage{algpseudocode}
\usepackage{booktabs}
\usepackage{graphicx}

\newtheorem{theorem}{Theorem}[section]
\newtheorem{lemma}[theorem]{Lemma}

\newtheorem{proposition}{Proposition}

\usepackage{xcolor}

\newcounter{alphasect}
\def\alphainsection{0}

\let\oldsection=\section
\def\section{%
  \ifnum\alphainsection=1%
    \addtocounter{alphasect}{1}
  \fi%
\oldsection}%

\renewcommand\thesection{%
  \ifnum\alphainsection=1%
    \Alph{alphasect}
  \else%
    \arabic{section}
  \fi%
}%

\newenvironment{alphasection}{%
  \ifnum\alphainsection=1%
    \errhelp={Let other blocks end at the beginning of the next block.}
    \errmessage{Nested Alpha section not allowed}
  \fi%
  \setcounter{alphasect}{0}
  \def\alphainsection{1}
}{%
  \setcounter{alphasect}{0}
  \def\alphainsection{0}
}%

\def\bbP{\mathbb{P}}

\def\bbR{\mathbb{R}}

\def\cC{\mathcal{C}}

\def\cQ{\mathcal{Q}}

\def\cX{\mathcal{X}}
\def\cY{\mathcal{Y}}

\def\exp{\mathrm{exp}}

\def\KL{\mathrm{KL}}

\def\norm#1{\left\|#1\right\|}